\documentclass{article} % For LaTeX2e
\usepackage{iclr2023_conference,times}

\usepackage{graphicx}
\usepackage{color}
\usepackage{amsmath,amssymb,amsthm,amsfonts}
\usepackage{subfigure}  %
\usepackage{mathrsfs}
\usepackage{dsfont}
\usepackage{longtable}
\usepackage{enumerate}

\usepackage{tkz-graph}
\usetikzlibrary{shapes.geometric}%
\tikzstyle{VertexStyle} = [shape = ellipse,
                           minimum width = 6ex,%
                           draw]
\tikzstyle{EdgeStyle}   = [->,>=stealth']

\usepackage{amsmath,amsfonts,bm}

\newtheorem{thm}{Theorem}[section]
\newcommand{\norm}[1]{\left\lVert#1\right\rVert}
\newtheorem{definition}{Definition}[section]

\newtheorem{lemma}[thm]{Lemma}

\newtheorem{theorem}[thm]{Theorem}

\newtheorem{assumption}{Assumption}[section]
\newtheorem{corollary}[thm]{Corollary}

\newcommand{\abs}[1]{\left\vert#1\right\vert}

\def\eqref#1{equation~\ref{#1}}
\def\1{\bm{1}}

\DeclareMathAlphabet{\mathsfit}{\encodingdefault}{\sfdefault}{m}{sl}
\SetMathAlphabet{\mathsfit}{bold}{\encodingdefault}{\sfdefault}{bx}{n}

\usepackage{hyperref}
\usepackage{url}

\title{Statistical Theory of Differentially Private Marginal-based Data Synthesis Algorithms}

\author{Ximing Li
 \thanks{Tsinghua University. Email:
\texttt{li-xm19@mails.tsinghua.edu.cn}.} \qquad \qquad
 Chendi Wang
\thanks{ Shenzhen Research Institute of Big Data \&
 Wharton Statistics and Data Science Department,
 University of Pennsylvania. Email:
 \texttt{chendi@wharton.upenn.edu}. }
\qquad \qquad
 Guang Cheng
 \thanks{
 Department of Statistics,
 University of California, Los Angeles.  Email:
\texttt{guangcheng@ucla.edu}.}
}

\iclrfinalcopy % Uncomment for camera-ready version, but NOT for submission.

\begin{document}

\maketitle

\begin{abstract}
 Marginal-based methods achieve promising performance in the synthetic data competition hosted by the National Institute of Standards and Technology (NIST).
 To deal with high-dimensional data, the distribution of synthetic data is represented by a probabilistic graphical model (e.g., a Bayesian network), while the raw data distribution is approximated by a collection of low-dimensional marginals.
 Differential privacy (DP) is guaranteed by introducing random noise to each low-dimensional marginal distribution.
 Despite its promising performance in practice, the statistical properties of marginal-based methods are rarely studied in the literature.
 In this paper, we study DP data synthesis algorithms based on Bayesian networks (BN) from a statistical perspective. We establish a rigorous accuracy guarantee for BN-based algorithms, where the errors are measured by the total variation (TV) distance or the $L^2$ distance.
 Related to downstream machine learning tasks, an upper bound for the utility error of the DP synthetic data is also derived. To complete the picture, we establish a lower bound for TV accuracy that holds for every $\epsilon$-DP synthetic data generator.
\end{abstract}

\section{Introduction}

In recent years, the problem of privacy-preserving data analysis has become increasingly important and \textit{differential privacy} \citep{dwork2006calibrating} appears as the foundation of data privacy.
Differential privacy (DP) techniques are widely adopted by industrial companies and the U.S.\ Census Bureau \citep{uber,rappor,samsung,CensusBureau2020,Census2018Abowd}.

One important method to protect data privacy is differentially private data synthesis (DPDS).
In the setting of DPDS, a synthetic dataset is generated by some DP data synthesis algorithms from a real dataset.
Then, one can release the synthetic dataset and the real dataset will be protected.
Recently, National Institutes of Standards and Technology (NIST) organized the differential privacy synthetic data competition \citep{url:nistunlinkable,url:nist2019,url:nist2020}.
In the NIST competition, the state-of-the-art algorithms are marginal-based \citep{mckenna2021winning}, where the synthetic dataset is drawn from a noisy marginal distribution estimated by the real dataset.
To deal with high-dimensional data, the distribution is usually modeled by the probabilistic graphical model (PGM) such as the Bayesian networks or Markov random fields \citep{jordan1999learning,wainwright2008graphical,zhang2017privbayes,mckenna2019graphical,cai2021privmrf}.

Despite its empirical success in releasing high-dimensional data, as far as we know, the theoretical guarantee of marginal-based DPDS approaches is rarely studied in literature.
In this paper, we focus on a DPDS algorithm based on the Bayesian networks (BN) known as the PrivBayes \citep{zhang2017privbayes} that is widely used in synthesizing sparse data (sparsity measured by the degree of a BN that will be defined later).
A BN is a directed acyclic graph where each vertex is a low-dimensional marginal distribution and each edge is the conditional distribution between two vertices.
It approximates the high-dimensional distribution of the raw data with a set of well-chosen low-dimensional distributions.
Random noise is added to each low-dimensional marginal to achieve differential privacy.
We aim to analyze the marginal-based approach from a statistical perspective and measure the accuracy of PrivBayes under different statistical distances including the total variation distance or the $L^2$ distance.

Another metric of synthetic data we are interested in is the utility metric related to downstream machine learning tasks.
Empirical evaluation of synthetic data in downstream machine learning tasks is widely studied in literature.
Existing utility metrics include Train on Synthetic data and Test on Real data (TSTR, \citep{esteban2017real}) and Synthetic Ranking Agreement  (SRA, \citep{jordon2018measuring}).
To our best knowledge, most of these utility evaluation methods are empirical without a theoretical guarantee.
Establishing the statistical learning theory of synthetic data is another concern of this paper. Precisely, we focus on the statistical theory of PrivBayes based on the TSTR error.

\textbf{Our contributions.}
Our contributions are three-fold. First, we theoretically analyze the marginal-based synthetic data generation and derive an upper bound on the TV distance and $L^2$ distance between real data and synthetic data.
The upper bounds show that the Bayesian network structure mitigates the ``curse of dimensionality".
An upper bound for the sparsity of real data is also derived from the accuracy bounds.
Second, we evaluate the utility of the synthetic data from downstream supervised learning tasks theoretically.
Precisely, we bound the TSTR error between the predictors trained on real data and synthetic data.
Third, we establish a lower bound for the TV distance between the synthetic data distribution and the real data distribution.

\subsection{Related Works and Comparisons}
Broadly speaking, our work is related to a vast body of work in differential privacy \citep{dinur2003revealing,dwork2004privacy,blum2005practical,dwork2007price,nissim2007smooth,barak2007privacy,mcsherry2007mechanism,machanavajjhala2008privacy,dwork2015efficient}. For example, \cite{mcsherry2007mechanism} proposed the exponential mechanism that is widely used in practice. \cite{machanavajjhala2008privacy} discussed privacy for histogram data by sampling from the perturbed cell probabilities. However, these methods are not efficient for releasing high-dimensional tabular data, since the domain size grows exponentially in the dimension (which is known as ``the curse of dimensionality").
%they require running time exponential in the dimension.
The state-of-art method for this problem is the marginal-based approach \citep{zhang2017privbayes,qardaji2014priview,zhang2021privsyn}. \cite{zhang2017privbayes} approximated the raw dataset by a sparse Bayesian network and then added noise to each vertex in the graph. \cite{zhang2021privsyn} selected a collection of 2-way marginals and a gradually updating method was applied to release synthetic data. Although most of them provide rigorous privacy guarantees, theoretical analysis on accuracy is rare.  \cite{wasserman2010stat} established a statistical framework of DP and derived the accuracy of distribution estimated by noisy histograms. Our setting is different from theirs. Precisely, we analyze how noise addition and post-processing affect the conditional distribution (Lemma \ref{conditional-lemma}). Moreover, our proof handles the non-trivial interaction between the Bayesian network and noise addition.
%new
% More recently, Boedihardjo1 et al. purpose several algorithms to answer linear queries, and most of them enjoy polynomial running time and rigorous accuracy guarantee. However, their results are limited to linear queries and are not applicable to TV-distance setting.  For example, [d-way-inf] achieves efficient accuracy in the $L^{\infty}$ metric. To maintain privacy, this algorithm adopts sub-sampling method. Precisely, it first draw $m$ samples from the uniform distribution on the entire domain and then re-weight the probabilities of the samples. Although quite impressive, when being applied to achieve TV-accuracy, this method requires exponentially large sub-sample size $m$.

Our lower bound (Theorem \ref{tv-lower}) is related to existing results of the worst case lower bounds under the DP constraint in literature \citep{hardt2010geometry,ullman2013answering,bassily2014private,steinke2017tight}. \cite{hardt2010geometry}  established lower bounds for the accuracy of answering linear queries with privacy budget $\epsilon$. \cite{ullman2013answering} derived the worst-case result that in general, it is NP-hard to release private synthetic data which accurately preserves all two-dimensional marginals. \cite{bassily2014private} built on their result and further developed lower bounds for the excess risk for every $(\epsilon,\delta)$-DP algorithm. Our result is novel since we consider private synthetic data and the corresponding TV accuracy. Existing results for linear quires are not directly applicable to TV accuracy since they heavily rely on the linear structure.

\section{Differential Privacy}
Differential privacy requires that any particular element in the raw dataset has a limited influence on the output \citep{dwork2006calibrating}. The definition is formalized as follows. Here the data domain is denoted as $\Omega$.

\begin{definition}[$(\epsilon,\delta)$-differential privacy]
Let $\mathcal{A}: \Omega^n \rightarrow \mathcal{R}$ be a randomized algorithm that takes a dataset of size $n$ as input, where the output space $\mathcal{R}$ is a probability space. For every $\epsilon,\delta\geq 0$, $\mathcal{A}$ satisfies $(\epsilon,\delta)$-differential privacy if for every two adjacent datasets $D_1$ and $D_2$, we have
\[
 \mathds{P}[\mathcal{A}(D_1)\in S]\leq \exp(\epsilon) \mathds{P}[\mathcal{A}(D_2)\in S]+\delta, \qquad \hbox{ for all measurable } S\subseteq \mathcal{R}.
\]
Here $D_1$ and $D_2$ are datasets of size $n$. We say that they are adjacent if they differ only on a single element, denoted as $D_1\simeq D_2$.
\end{definition}
For $\delta=0$, we abbreviate the definition as $\epsilon$-differential privacy ($\epsilon$-DP). A widely used meta-mechanism to ensure $\epsilon$-DP is the Laplace mechanism. The Laplace mechanism privatizes a function $f$ on the dataset $D$ by adding i.i.d.\ Laplace noises (denoted as $\eta\sim $ $\mathrm{Lap}(\lambda)$ ) to each output value of $f(D)$. Here the probability density function of $\eta$ is given by
$
\mathds{P}[\eta=x]=\frac{1}{2\lambda}\exp(\frac{-\abs{x}}{\lambda}).
$
  \cite{dwork2004privacy} show that it ensures $\epsilon$-DP when $\lambda\geq \Delta_f/\epsilon$, where $\Delta_f$ is the $L^{1}$ sensitivity of $f$:
\[
\Delta_f = \max_{(D_1,D_2): D_1\simeq D_2} \norm{f(D_1)-f(D_2)}_{1}.
\]

\section{Marginal-based Data Synthesis Algorithms}
In this section, we introduce DP marginal-based methods. For simplicity, we consider Boolean data where $\Omega = \{0,1\}^d$ and $|\Omega| = 2^d$.
It's obvious that our theory can be generalized to any categorical dataset with a finite domain size.

\subsection{Differentially Private Estimate of Low-dimensional Marginal Distributions}
Given a dataset $D = \{x^{(i)}\}_{i=1}^n\subset \Omega^n$ drawn independently from a distribution, the probability mass function is estimated by
\begin{align}\label{eq:DefnMarginal}
    p_{D}(x) = \frac{1}{n}\sum_{i=1}^n\mathds{1}[x^{(i)} = x], \qquad \hbox{ for all } x\in\Omega.
\end{align}

\textbf{Noise addition and post processing.}
We then sanitize $p_{D}(x)$ by the Laplace mechanism.
Note that the sensitivity of $p_{D}(x)$ is $1/n$.
Then, we define $\widetilde{p}_D = p_D + \mathrm{Lap}(1/(n\epsilon))$ and $\widetilde{p}_D$ is $\epsilon$-DP.
Adding noise leads to inconsistency. To be specific, some estimated probabilities may be negative and the overall summation may not be 1. The following two kinds of post processing methods to address the inconsistency are widely adopted in marginal-based methods (cf., \citep{mckenna2019graphical,zhang2017privbayes}).

\textit{Normalization.}
We convert all the negative probabilities to zeros, and then normalize all the probabilities by a scalar such that their summation is 1.

\textit{$L^2$-projection.}
We project the inconsistent distribution onto the probability simplex using the $L^2$ metric. Specifically, for an inconsistent distribution $(a_1,\cdots, a_m)$, the output is
\[
(b_1,\cdots,b_m):=\mathop{\arg \min} \limits_{\widetilde{b}_i\geq 0,\, \sum \widetilde{b}_i=1} \sum_{i=1}^m (a_i-\widetilde{b}_i)^2.
\]
% For any $i$, we denote the result of the marginal $\mathds{P}[x_i, \Pi_i]$ after noise addition and post process as $\widehat{\mathds{P}}[x_i,\Pi_i]$.

\subsection{Marginal Selection and Bayesian Networks}

It is well-known that marginal-based methods have the curse of dimensionality. One way to mitigate the curse of dimensionality is adopting Bayesian networks \citep{zhang2017privbayes}.
% One promising way to release differentially private high dimensional data is the marginal-based approach. PrivBayes follows this framework, which consists of the following steps \citep{zhang2017privbayes}.

\textbf{Marginal selection.}
We first disassemble the raw dataset into a group of lower dimensional marginal datasets. Precisely, PrivBayes \citep{zhang2017privbayes} uses a sparse Bayesian network $\{x_1,\cdots, x_d\}$ to approximate the raw data. Each node $x_i$ corresponds to an attribute, and each edge from $x_j$ to $x_i$ represents $\mathds{P}[x_i \mid  x_j]$, which is the probability of $x_j$ causing $x_i$. We denote $\Pi_i:=\{j\mid x_j \rightarrow x_i\}$, which is the collection of all the attributes that affect $x_i$. \cite{zhang2017privbayes} also make the following assumptions on the network structure. Here $k$ is a pre-fixed parameter that is much smaller than $d$.
\begin{assumption}[Sparsity]\label{sparsity}
    The degree of the Bayesian network is no more than $k$. Precisely, for any $i$, the size of $\Pi_i$ is no more than $k$.
\end{assumption}
The second assumption ensures that the graph cannot contain loops, which aids sampling from the graph.
\begin{assumption}\label{non-cycle}
     For any $i$, we have $\Pi_i\subset \{x_1,\cdots, x_{i-1}\}$.
\end{assumption}

For example, the Bayesian network in Figure \ref{graph} satisfies Assumption \ref{sparsity} for $k=2$ and Assumption \ref{non-cycle}. The joint distribution is $\mathds{P}(5\mid 4,3)\mathds{P}(4\mid 3,2)\mathds{P}(3\mid 2,1)\mathds{P}(2\mid 1)$.

\begin{figure}[ht]
\centering
    \begin{tikzpicture}[scale=1]
    \SetGraphUnit{2}
     \Vertex{5}  \WE(5){4}   \NO(5){3}
     \NO(4){2} \NO(2){1}
     \Edges(2,4,5) \Edges(3,4,5) \Edges(1,3,5)
     \Edges(2,3,5) \Edges(1,2,4)
    \end{tikzpicture}
\caption{A Bayesian network over 5 attributes of degree 2}
\label{graph}
\end{figure}
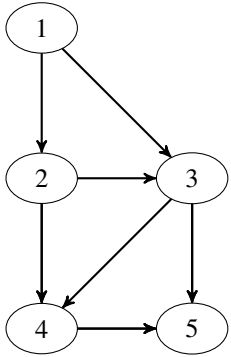

\textbf{DP Bayesian networks.}
In a Bayesian network, each low-dimensional marginal distribution $\mathds{P}(x_i,\Pi_i)$ is estimated by the marginal function defined by (\ref{eq:DefnMarginal}).
For privacy consideration, we add Laplace noise $\mathrm{Lap}(d/n \epsilon)$ to the marginal $\mathds{P}(x_i,\Pi_i)$ and obtain the DP distribution $\widehat{\mathds{P}}[x_i,\Pi_i]$ by using post processing to the noisy marginal.
Then the overall privacy budget can be calculated by the composition property of DP \citep{dwork2008differential,zhang2017privbayes} and is $\epsilon.$

\textbf{Generating synthetic data.}
The loop-free Bayesian network provides an efficient sampling approach. Precisely, we draw $x_i$ from $\widehat{\mathds{P}}[x_i\mid \Pi_i]$ in an increasing order of $i$. Recall that Assumption \ref{non-cycle} ensures that $x_j\notin \Pi_i$ for any $j$\textgreater $i$. Therefore, by the time $x_i$ is to be sampled, all nodes in $\Pi_i$ must have been sampled. This verifies that the sampling approach is practical. Moreover, sampling $x_i$ only needs the marginal $\widehat{\mathds{P}}[x_i,\Pi_i]$, instead of the full distribution. By Assumption \ref{sparsity}, it is a marginal with size less than $k+1$. This leads to a small computational load since $k$ is small. With this sampling method, one \citep{zhang2017privbayes} can show that the synthetic private distribution is $\widehat{\mathds{P}}[x_1,\cdots,x_d]=\prod_{i=1}^d \widehat{\mathds{P}}[x_i\mid \Pi_i]$.
\section{Accuracy of PrivBayes}
In this section, we develop some theoretical results of the accuracy of PrivBayes. We discuss the proof of these results briefly in Section \ref{proof sketch}.

    % \begin{tikzpicture}[scale=1]
    % \SetGraphUnit{2}
    % \Vertex{11}  \NO(11){5}   \WE(11){2} \SO(11){7} \EA(11){10}
    %              \SOEA(11){9} \SOEA(10){3}\SOEA(9){8}
    % \Edges(5,11,10) \Edges(3,8,9) \Edges(7,11,2)
    % \Edges(11,9)    \Edges(3,10)
    % \end{tikzpicture}

\subsection{Statistical Distances}
The goal of this subsection is to establish the accuracy guarantee for PrivBayes with different post-processing methods: normalization and $L^2$-projection. By the term ``accuracy", we mean the TV or the $L^2$ distance between the synthetic distribution and the raw data distribution, respectively. Note that the error comes from two sources: 1) approximating the raw data by a Bayesian network that satisfies Assumption \ref{sparsity} and Assumption \ref{non-cycle} , 2) adding noise and post processing. The first error, however, only relies on the sparsity of the raw data. Since we aim to establish our result for general raw data, we only focus on the second one. Therefore, it is natural for us to make the following assumption.
\begin{assumption}\label{structure}
We assume the raw data distribution can be represented by a Bayesian network with $d$ vertices that satisfies Assumption \ref{sparsity} and Assumption \ref{non-cycle}.
\end{assumption}
With this assumption, PrivBayes (normalization) has the following accuracy guarantee.
\begin{theorem}\label{PrivBayes tv}
Assuming that the raw dataset $\mathds{D}$ is Boolean and satisfies Assumption \ref{structure}, then we have
\[
\norm{\widehat{\mathds{P}}-\mathds{P}}_{\mathrm{TV}}\leq \frac{12 d^2 2^{2k} (k+1)}{n \epsilon}\log \frac{2d}{\delta}
,
\]
with probability at least $1-\delta$  (with respect to the randomness of the Laplace mechanism and the same below). Here $\mathds{P}$ is the empirical distribution of $\mathds{D}$ and $\widehat{\mathds{P}}$ is the output of PrivBayes (normalization) with privacy budget $\epsilon$.
\end{theorem}
 \begin{proof}
 See Section \ref{proof sketch} for a proof sketch.
 \end{proof}
The other post-processing method ($L^2$-projection) we studied is also efficient. The following result verifies that PrivBayes ($L^2$-projection) enjoys the similar accuracy guarantee in terms of the $L^2$ distance.
Here, with a little bit abuse of notations, we denote the $L^2$ distance between two distributions as the $L^2$ distance between their density functions.
\begin{theorem}\label{PrivBayes L2}
Assuming the raw dataset $\mathds{D}$ is Boolean and Assumption \ref{structure} is satisfied, then we have
\[
\norm{\widehat{\mathds{P}}-\mathds{P}}_{L^2}\leq \frac{12 d^2  2^{k} (k+1)}{n \epsilon}\log \frac{2d}{\delta}
,
\]
with probability at least $1-\delta$. Here $\mathds{P}$ is the empirical distribution of $\mathds{D}$ and $\widehat{\mathds{P}}$ is the output of PrivBayes ($L^2$-projection) with privacy budget $\epsilon$.
\end{theorem}
\begin{proof}
The proof is similar to Theorem \ref{PrivBayes tv}. However it is still non-trivial and we discuss it in detail in Appendix \ref{L2-proof}.
\end{proof}
\textbf{Discussion.} Theorem \ref{PrivBayes tv} and Theorem \ref{PrivBayes L2} achieves a bound that is consistent (tends to 0 as $n$ tends to infinity) if $k\ll \log_2 (n\epsilon/d^2)$.
Moreover, a smaller $k$ leads to smaller upper bounds.
Since we assume that the real data and synthetic data share the same $k$, we conclude that PrivBayes achieve better performance on sparser real datasets.
% To ensure the consistency of $\widehat{P}$, we require $k\leq \log_2 (n\epsilon/d^2)$,  and
The size of $k$ is often rather small in real application. For example, \cite{zhang2017privbayes} chose $k\leq4$ in the simulation. Theorem \ref{PrivBayes tv} also characterizes the reliance on the privacy budget $\epsilon$. Precisely, tighter privacy budget means better privacy guarantee, but leads to worse performance. Moreover, our rate is polynomial in the dimension $d$. Comparing with directly applying the Laplace mechanism to the whole domain (see Theorem \ref{laplace}), our result shows that by heavily deploying the network structure, the Bayesian network exponentially refines the rate.

% \cite{wasserman2010stat} also discussed the normalization process, our analysis is different from theirs. We analyze how noise addition and post processing affect the conditional distribution (Lemma \ref{conditional-lemma}). Moreover, our proof handles the non-trivial interaction between the Bayesian network and noise addition. By heavily deploying the network structure, we exponentially refine the rate.

\begin{theorem}\label{laplace}
Assuming that the raw dataset $\mathds{D}$ is Boolean, then we have
\[
\norm{\widetilde{\mathds{P}}-\mathds{P}}_{\mathrm{TV}}\leq \frac{d 2^{2d}}{n\epsilon} \log \frac{2}{\delta},
\]
with probability at least $1-\delta$. Here $\widetilde{\mathds{P}}$ is the synthetic distribution generated by directly applying Laplace mechanism to the entire domain.
\end{theorem}
\begin{proof}

By the definition of Laplace  mechanism, we add i.i.d.\ ${\mathrm{Lap}}(1/n \epsilon)$ noise to all the $2^d$ choices of $\mathds{P}(x_1,\cdots,x_d)$ and then normalize them.Then, Theorem \ref{laplace} can be proved similarly as Lemma \ref{noisy marginals} with $m=2^d$. %Then by setting $m=2^d$ in Lemma \ref{noisy-marginals}, we prove Theorem \ref{laplace}.
\end{proof}

\subsection{Utility Errors}
 In real world practice, the raw training data in supervised learning may contain sensitive information, like personal preference of users. Therefore, it is not allowed to be released to the public. An alternative is to release differentially private synthetic data instead of the raw data. A central problem is the ``utility" of synthetic training data, which means the evaluation of synthetic data in downstream tasks. We explain the term utility in detail below.

Consider a dataset $\mathds{D}=\{x^{(i)}\}_{i=1}^n$ drawn from the domain $\Omega$ with sample size $n$. We denote its corresponding synthetic dataset as $\widehat{\mathds{D}}=\{\widehat{x}^{(i)}\}_{i=1}^{\widehat{n}}$ of size $\widehat{n}$. We use the empirical risk minimization (ERM) model to capture the supervised learning. Precisely, the ERM estimators for the raw dataset $\mathds{D}$  and the synthetic data $\widehat{\mathds{D}}$ are defined as
\begin{align}\label{erm-define}
\widehat{\theta}=\mathop{\arg\min}\limits_{\theta\in \mathcal{C}} \frac{1}{n} \sum_{i=1}^n \ell(\theta,x^{(i)})+\lambda J(\theta),\nonumber \\
\widehat{\theta}_{\mathrm{syn}}=\mathop{\arg\min}\limits_{\theta\in \mathcal{C}} \frac{1}{\widehat{n}} \sum_{i=1}^{\widehat{n}} \ell(\theta,\widehat{x}^{(i)})+\lambda J(\theta),
\end{align}
respectively. Here $\mathcal{C}$ is a convex closed set. We assume the loss function $\ell(\cdot,x)$ is convex on $\mathcal{C}$ and is $L$-Lipschitz in $x$ for some $L\geq 0$. The regularization term $J(\cdot)$ is adopted to prevent over-fitting. The model captures a wide range of applications. For example, given a data point $x^{(i)}=(u_i,v_i)\in\{0,1\}^{d+1}$, by defining the hinge loss $\ell(\theta,x^{(i)})=(1-\left\langle\theta,u_i\right\rangle \cdot v_i)_+$, we recover the popular support vector machine (SVM) classifier. The loss is $\sqrt{d+1}$-Lipschitz in $\theta$ since $\norm{x^{(i)}}_2\leq \sqrt{d+1}$.

The utility of the synthetic dataset $\widehat{\mathds{D}}$ measures whether $\widehat{\theta}_{\mathrm{syn}}$ and $\widehat{\theta}$ perform similarly on the prediction task \citep{esteban2017real}. To be specific, the following metric is used to evaluate the utility,
\begin{equation}\label{utility-define}
U(\widehat{\mathds{D}},{\mathds{D}}):=\frac{1}{n} \abs{R(\widehat{\theta})-R(\widehat{\theta}_{\mathrm{syn}})},
\end{equation}
where $R(\theta)=\sum_{i=1}^n \ell(\theta,x^{(i)})$ is the empirical risk on $\mathds{D}$ (\citep{rankin2020reliability,hittmeir2019utility}). Intuitively, the asymptotic behavior of $U(\widehat{\mathds{D}},{\mathds{D}})$ is affected by the difference between distributions of synthetic data and true data. This fact is characterized in Theorem \ref{PrivBayes utility}.

We first make the following assumption on the bound of the loss function $\ell(\cdot,\cdot)$, which is quite natural due to its continuity \citep{bassily2014private}.
\begin{assumption}\label{boundedness}
For any $\theta\in \mathcal{C}$ and any data point $x$ in $\Omega$, we have $\abs{\ell(\theta,x)}\leq1$.
\end{assumption}

\textbf{Generating synthetic dataset from PrivBayes.} We still denote the raw dataset as $\mathds{D}$ and denote  $\mathds{P}$ its empirical distribution. Its corresponding output of PrivBayes is a distribution denoted as $\mathds{Q}$. To generate the synthetic training data $\widehat{\mathds{D}}$, we draw $\widehat{n}$ i.i.d.\ samples from $\mathds{Q}$. The corresponding empirical distribution is denoted as $\widehat{\mathds{Q}}$.

With the above preparation, we are now ready to state our result that characterizes the utility of PrivBayes.
\begin{theorem}\label{PrivBayes utility}
If Assumption \ref{structure} and Assumption \ref{boundedness} hold and the raw dataset $\mathds{D}$ is Boolean, then we have
\begin{align}\label{PrivBayes utility1}
U(\widehat{\mathds{D}},\mathds{D})
&\leq C(\lambda)+C_1 \left\| \mathds{Q} -  \mathds{P}\right\|_{\mathrm{TV}}+2\mathcal{R}_{C}+\sqrt{\frac{\log \frac1\delta}{2\widehat{n}}}\nonumber
\\
&\leq C(\lambda)+C_1 \frac{2^{2k} d^2(k+1)} {n\epsilon}\ln \frac{2 d}{\delta}+2\mathcal{R}_{C}+\sqrt{\frac{\log \frac1\delta}{2\widehat{n}}},
\end{align}
with probability at least $1-\delta$. Here $\mathcal{R}_C$ is the Rademacher complexity of the function class $\{x\mapsto \ell(\theta,x) \mid \theta\in \mathcal{C}\}$ and $C_1$ is a positive universal constant. The term $C(\lambda)$ is non-negative and vanishes when $\lambda=0$, namely $C(0)=0$.
\end{theorem}
\begin{proof}
 See Section \ref{proof sketch} for a proof sketch.
 \end{proof}
\textbf{Discussion.} The term $C(\lambda)$ comes from the regularization process.
$C(\lambda)=0$ if no regularization is applied ($\lambda=0$).
%It vanishes when no regularization is applied.
In real practice, $\lambda$ is often much smaller than $d^2/n\epsilon$. Therefore $C(\lambda)$ is also relatively small. The Rademacher complexity in equation (\ref{PrivBayes utility1}) comes from the sampling process.

Our result then implies that, when the sample size $n$ and $\widehat{n}$ are sufficiently large and the regularization parameter $\lambda$ is sufficiently small, the quality loss caused by the private mechanism is rather small. In other words, private synthetic data generated by PrivBayes performs similarly to raw data in downstream learning tasks.

\section{Lower Bound}
In this section, we complete the picture by deriving a lower bound for the TV-distance between synthetic private distribution and the raw data distribution.

\textbf{Notations and conventions.}
As before, the raw dataset $\mathds{D}$ is of size $n$ with empricial distribution $\mathds{P}$. The data domain is denoted as $\Omega$. A synthetic data generator is a randomized algorithm that sends a dataset of size $n$ to a distribution over $\Omega$. We also need the following assumption on the range of the parameters.
\begin{assumption}\label{range}
We assume that $d/\epsilon \ll n \ll \abs{\Omega}$.
\end{assumption}
The first part of this assumption allows a rather wide choice of $\epsilon$ in practice. For instance, in two real datasets ACS \citep{ruggles2015integrated} and Adult \citep{bache2013uci}, the size $n\approx 40,000$, the dimension $d\approx 40$. Then Assumption \ref{range} only requires $\epsilon \geq 1/1,000$. Moreover, the size of $\Omega$ is at least $2^d$, which is clearly much larger than $n$. Therefore, the second part of Assumption \ref{range} holds for real world datasets.

We now state Theorem \ref{tv-lower} that establishes the lower bound for TV-distance.
\begin{theorem}\label{tv-lower}
If Assumption \ref{range} holds, then for any synthetic data generator $A(\cdot)$ with privacy budget $\epsilon$, and for any $0\leq\delta\leq 1/2$, there exists a dataset $\mathds{D}$ of size $n$, such that
\[
\norm{A(\mathds{D})-\mathds{P}}_{\mathrm{TV}}\geq \frac{1 }{n\epsilon}\log (\delta|\Omega|)
\]
with probability at least $1-2\delta$. Here $\mathds{P}$ is the empirical distribution of $\mathds{D}$.
\end{theorem}
\begin{proof}
See Appendix \ref{lower-proof} for a detialed proof.
\end{proof}

 Choosing $\delta=1/4$ in Theorem \ref{tv-lower} yields the following corollary.
\begin{corollary}
If $\abs{\Omega}\geq 4\exp(d)$, then for any synthetic data generator $A(\cdot)$ with privacy budget $\epsilon$, there exists a dataset $\mathds{D}$ of size $n$ such that
\[
\norm{A(\mathds{D})-\mathds{P}}_{\mathrm{TV}}\geq \frac{d}{n\epsilon}
\]
with probability at least $1/2$. Here $\mathds{P}$ is the empirical distribution of $\mathds{D}$.
\end{corollary}

\textbf{Discussion and comparison.}
Comparison with the upper bound in Theorem \ref{PrivBayes tv},  PrivBayes is sub-optimal up to a $d$ factor.
The sub-optimality is caused by the composition property of DP (the dataset is processed $d$ times in a Bayesian network) and the structure of a Bayesian network.

% The worst-case lower bounds under the DP constraint are studied in many existing papers.
% For example, \cite{hardt2010geometry} applied packing technique to linear queries and derived lower bounds for differentially private synthetic data.
% However, their results are not directly applicable to TV accuracy studied in as we studied in Theorem \ref{tv-lower} since they heavily rely on the linear structure.

% \cite{hardt2010geometry} applied packing technique and derived lower bounds for differentially private synthetic data, their analysis is limited to linear queries. Our result, however, establishes lower bound for TV distance. Their results are not directly applicable to TV accuracy since they heavily rely on the linear structure.

\section{Proof Sketch for the Technical Results}\label{proof sketch}
\subsection{Proof Sketch for Theorem \ref{PrivBayes tv}}
We begin with a technical lemma that characterizes the normalization process. See Appendix \ref{tv-proof} for its detailed proof.
\begin{lemma}\label{noisy marginals}
For a distribution $(a_1,\cdots,a_m)$, we denote its outcome after adding i.i.d.\ ${\mathrm{Lap}}(d/n\epsilon)$ noise and normalizing it as $(b_1,\cdots,b_m)$. Then, for all large $n$ and all $\delta>0$, it holds that
\[
\max _i\abs{a_i-b_i}\leq \frac{3md}{n\epsilon}\log \frac{m}{\delta},
\]
with probability at least $1-\delta$.
\end{lemma}
Lemma \ref{noisy marginals} characterizes the difference between $\widehat{\mathds{P}}(x_i,\Pi_i)$ and $\mathds{P}(x_i,\Pi_i)$. However, we need further analysis to establish the conditional version of Lemma \ref{noisy marginals}. To be specific, we need to bound $\abs{\widehat{\mathds{P}}(x_i\mid \Pi_i)-\mathds{P}(x_i\mid\Pi_i)}$. The following result serves for this goal.
\begin{lemma}\label{conditional-lemma}
Consider two non-negative real vectors $(a_1,\cdots,a_s)$ and $(b_1,\cdots,b_s)$  (not necessary to be distributions). If, for some $\beta\geq 0$, we have
\begin{equation}\label{conditional 2}
\max_{j}\abs{a_j-b_j}\leq \beta,
\end{equation}
then, for any $l\in \{1,\cdots,s\}$, the following result holds.
\begin{equation}\label{conditional}
\abs{\frac{a_l}{\sum_{j=1}^s a_j}-\frac{b_l}{\sum_{j=1}^s b_j}}\leq \frac{s\beta}{\sum_{j=1}^s b_j}.
\end{equation}
\end{lemma}
\begin{proof}
See Appendix \ref{tv-proof} for a detailed proof.
\end{proof}
Combining Lemma \ref{noisy marginals} and Lemma \ref{conditional-lemma}, the distance between the conditional distributions is bounded in the following result.
\begin{lemma}\label{conditional-tv}
If $\mathds{D}$ is boolean and satisfies Assumption \ref{structure}, then we have
\begin{equation}\label{dominator}
\abs{\widehat{\mathds{P}}(x_i\mid \Pi_i)-\mathds{P}(x_i\mid \Pi_i)}
\leq \frac{6 d 2^k(k+1)}{n \epsilon}\log \frac{2}{\delta}\frac{1}{\mathds{P}(\Pi_i)},
\end{equation}
with probability at least $1-\delta$, simultaneously for all $i$ and all choices of $(x_i,\Pi_i)$.
\end{lemma}
\begin{proof}
Setting
\[
\begin{cases}
&m=2^{k+1},\\
&s=2,\\
&a_1=\widehat{\mathds{P}}(1, \Pi_i),\, a_2=\widehat{\mathds{P}}(0, \Pi_i), \\ &b_1={\mathds{P}}(1, \Pi_i), \, b_2={\mathds{P}}(0, \Pi_i),\\
& \beta=
\frac{3md}{n\epsilon}\log \frac{m}{\delta},
\end{cases}
\]
in Lemma \ref{noisy marginals} and Lemma \ref{conditional-lemma} concludes the proof.
\end{proof}

To bound the TV-distance, we begin with rewriting it in telescoping series and applying Lemma \ref{conditional-tv}. One technical impediment for estimation is the fraction term $1/\mathds{P}(\Pi_i)$ in (\ref{dominator}). To address this challenge, we need to deploy the Bayesian network structure (Assumption \ref{sparsity} and Assumption \ref{non-cycle}). Deploying the network structure is quite technical and lengthy, we defer the detail to Appendix \ref{tv-proof}.

\subsection{Proof Sketch for Theorem \ref{PrivBayes utility}}
We begin with some notations. The non-regularized estimators trained on $\mathds{D}$ and $\widehat{\mathds{D}}$ by an ERM model are denoted as $\theta^*$ and $\theta^*_{\mathrm{syn}}$. Formally, we define
\begin{align*}
\theta^{*}:= \mathop{\arg\min}\limits_{\theta\in \mathcal{C}} \frac{1}{n} \sum_{i=1}^n \ell(\theta, x^{(i)}),\\
\theta^{*}_{\mathrm{syn}}:=\mathop{\arg\min}\limits_{\theta\in \mathcal{C}} \frac{1}{\widehat{n}} \sum_{i=1}^{\widehat{n}} \ell(\theta,\widehat{x}^{(i)}).
\end{align*}
We further define the prediction risk with respect to a certain distribution. For any distribution on $\Omega$, denoted as $P$, and any $\theta\in \mathcal{C}$ we define
\begin{equation}\label{r-define}
R(\theta,P):=\sum_{x\in \Omega}\ell(\theta,x)P(x)
\end{equation}
as the prediction risk with respect to $P$. Then $R(\cdot)$ in (\ref{utility-define}) is equal to $R(\cdot,\mathds{P})$.

We are now ready to sketch the proof. The most important step of the proof is to decompose the utility in (\ref{utility-define} into the following seven terms
\begin{align}\label{decompose}
&U(\widehat{\mathds{Q}},\mathds{P})\leq\nonumber \\
&\qquad \underbrace{\abs{R(\widehat{\theta}_{\mathrm{syn}}, \mathds{P})-R(\widehat{\theta}_{\mathrm{syn}},\mathds{Q})}}_{\textrm{term (i)}}
+\underbrace{\abs{R(\widehat{\theta}_{\mathrm{syn}},\mathds{Q})-R(\widehat{\theta}_{\mathrm{syn}}, \widehat{\mathds{Q}})}}_{\textrm{term (ii)}}+\underbrace{\abs{R(\widehat{\theta}_{\mathrm{syn}}, \widehat{\mathds{Q}})-R(\theta^*_{\mathrm{syn}},\widehat{\mathds{Q}})}}_{\textrm{term (iii)}}
\nonumber \\
&\qquad
+\underbrace{\abs{R(\theta^*_{\mathrm{syn}},\widehat{\mathds{Q}})-R(\theta^*,\widehat{\mathds{Q}})}}_{\textrm{term (iv)}}+\underbrace{\abs{R(\theta^*,\widehat{\mathds{Q}})-R(\theta^*,\mathds{Q})}}_{\textrm{term (v)}}+\underbrace{\abs{R(\theta^*,\mathds{Q})-R(\theta^*,\mathds{P})}}_{\textrm{term (vi)}}\nonumber \\
&\qquad+\underbrace{\abs{R(\theta^*,\mathds{P})-R(\widehat{\theta},\mathds{P})}}_{\textrm{term (vii)}}.
\end{align}
Recall that $\mathds{Q}$ is the output of PrivBayes and $\widehat{\mathds{Q}}$ is the empirical distribution of the $\widehat{n}$ samples drawn independently from $\mathds{Q}$. Here \textrm{term (i)} and \textrm{term (vi)} come from the difference between the synthetic distribution $\mathds{Q}$ and the raw distribution $\mathds{P}$. They can be bounded above by the distance between the synthetic distribution and the raw one. Term (ii) and term (v) come from sampling and are bounded by classical Rademacher method. Term (iii) and term (vii) are derived from the regularization process. They combine to be the $C(\lambda)$ term. Bounding term (iv), however, is more tricky and requires more detailed analysis. We discuss each group in detail in Appendix \ref{utility-proof}.
\section{Discussions and Future Topics}
We establish perhaps the first statistical analysis for the accuracy and utility of Bayesian network-based data synthesis algorithms. We also derive a lower bound for the accuracy to complete the picture. Compared with the lower bound, the accuracy bound we achieve is sub-optimal up to a $d$ factor.
One way to improve the accuracy is to reduce the effects of random noise in releasing the synthetic data through some post-processing procedures. However, it is still quite challenging to develop a practical algorithm based on this idea, and we leave it for future work.
\section*{Acknowledgments}
   We appreciate Prof.\ Ninghui Li and Dr.\ Zitao Li for their discussions about the background and applications of marignal-based data synthesis methods, which motivates us to study the corresponding theory. This research is supported by the Office of Naval Research [ONR N00014-22-1-2680] and the National Science Foundation [NSF -- SCALE MoDL (2134209)].
\newpage

\bibliography{iclr2023_conference}
\bibliographystyle{iclr2023_conference}

\newpage

% \section*{References}

% References follow the acknowledgments. Use unnumbered first-level heading for
% the references. Any choice of citation style is acceptable as long as you are
% consistent. It is permissible to reduce the font size to \verb+small+ (9 point)
% when listing the references.
% Note that the Reference section does not count towards the page limit.
% \medskip

% \bibliographystyle{plain}

% \bibliography{Ref_NeurIPS}

\newpage
\appendix
\section{Proof of Theorem \ref{PrivBayes tv}}\label{tv-proof}
\subsection{Proof of Lemma \ref{noisy marginals}}
\begin{proof}
We denote the outcome after adding i.i.d.\ Laplace noise as $(a_1+v_1,\cdots, a_m+v_m)$, where $v_i\sim$ ${\mathrm{Lap}}(d/n\epsilon)$. We further define
\[
S=\sum_{i=1}^m (a_i+v_i)_+.
\]
Here $(\cdot)_+=\max \{\cdot,0\}$. Then by the definition of normalization, we have $b_i= (a_i+v_i)_+/S$. Note that
\[
(a_i+v_i)_+-a_i=
\begin{cases}
v_i & \text{ if } a_i+v_i\geq 0,\\
-a_i & \text{ if } a_i+v_i\leq 0.
\end{cases}
\]
Since $a_i\geq 0$, in both cases it holds that
\begin{equation}\label{noisy marginals 1}
\abs{(a_i+v_i)_+-a_i}\leq \abs{v_i}.
\end{equation}
Combining this fact with $\sum_i a_i=1$ shows that
\begin{equation}\label{noisy marginals 2}
\abs{S-1}\leq \sum_{i=1}^m \abs{v_i}.
\end{equation}
 A simple union bound over $m$ Laplace random variables shows that
\[
\max _i\abs{v_i}\leq \frac{d}{n\epsilon} \log (\frac{m}{\delta}):=M,
\]
with probability at least $1-\delta$. Therefore with probability at least $1-\delta$ it holds that
\begin{equation}\label{s}
\abs{S-1}\leq mM.
\end{equation}
With (\ref{s}), the quantity $\abs{a_i-b_i}$ is bounded as
\begin{align}\label{noisy marginals 3}
&\abs{\frac{(a_i+v_i)_+}{S}-a_i}\leq \abs{\frac{(a_i+v_i)_+}{S}-\frac{a_i}{S}}+\abs{\frac{a_i}{S}-a_i} \nonumber\\
&\qquad \leq \frac{\abs{v_i}}{S}+\frac{\abs{a_i}\abs{1-S}}{S} \nonumber\\
&\qquad \leq \frac{\abs{v_i}}{S}+\frac{\abs{1-S}}{S}.
\end{align}
Here (\ref{s}) ensures that $S\neq 0$ with high probability for all large $n$, and the last line is due to $0\leq a_i\leq 1$. Further applying (\ref{s}) yields that
\begin{align*}
\abs{\frac{(a_i+v_i)_+}{S}-a_i}\leq \frac{M+mM}{1-mM}\leq 2(1+m)M
\end{align*}
for all large $n$, with probability at least $1-\delta$. This concludes the proof.
\end{proof}
\subsection{Proof of Lemma \ref{conditional-lemma}}
\begin{proof}
The left-hand side of (\ref{conditional}) can be rewritten as
\begin{equation}\label{conditional1}
\frac{\abs{a_l \sum_{j=1}^s b_j-b_l \sum_{j=1}^s a_j}}{(\sum_{j=1}^s a_j)(\sum_{j=1}^s b_j)}.
\end{equation}
We denote the numerator as $S_1$. Then by adding and contracting $\sum_{j=1}^s a_l a_j$, we rewrite and bound $S_1$ as
\begin{align}\label{conditional 3}
&\abs{\sum_{j=1,j\neq l}^s a_l(b_j-a_j)+a_j(a_l-b_l)}\nonumber \\
&\qquad\leq (s-1)\beta a_l+ \beta \sum_{j=1,j\neq l}^s a_j \nonumber\\
&\qquad\leq s\beta \sum_{j=1}^s a_j.
\end{align}
Here the second line is due to (\ref{conditional 2}) and the fact that $a_j$ is non-negative for all $j$. Then Lemma \ref{conditional-lemma} is now proved by (\ref{conditional1}) and (\ref{conditional 3}).
\end{proof}
\subsection{Proof of Theorem \ref{PrivBayes tv}}
\begin{proof}
By Assumption \ref{structure} we have
\begin{align}\label{tv-1}
&\norm{\widehat{\mathds{P}}-\mathds{P}}_{\mathrm{TV}}\nonumber \\
&\qquad =\sum_{x_1=0,1}\ldots\sum_{x_d=0,1}
\underbrace{\abs{\prod_{i=1}^d \widehat{\mathds{P}}(x_i\mid \Pi_i)-\prod_{i=1}^d\mathds{P}(x_i\mid\Pi_i)}}_{ \textrm{term (i)}}.
\end{align}
By adding and subtracting  $d-1$ terms, we rewrite term (i) in the parentheses as
\begin{align}\label{tv-2}
&\textrm{Term (i)}\nonumber \\
&\qquad \leq\sum_{j=1}^d \left[ \prod_{i=1}^{j-1} {\mathds{P}}(x_i \mid \Pi_i)\right] \abs{\widehat{\mathds{P}}(x_j\mid \Pi_j)-\mathds{P}(x_j \mid \Pi_j)}
\left[\prod_{i=j+1}^d \widehat{\mathds{P}}(x_i \mid \Pi_i)\right].
\end{align}
Combining (\ref{tv-1}), (\ref{tv-2}) and Lemma \ref{conditional-tv} implies
\begin{align}\label{tv-3}
&\norm{\widehat{\mathds{P}}-\mathds{P}}_{\mathrm{TV}}\nonumber \\
&\qquad \leq \sum_{x_1=0,1}\ldots\sum_{x_d=0,1}\sum_{j=1}^d \left[\prod_{i=1}^{j-1} {\mathds{P}}(x_i \mid \Pi_i)\right] \frac{6 2^{k}  d (k+1)}{n \epsilon}\log \frac{2d}{\delta}\nonumber \\
&\qquad \times \frac{1}{\mathds{P}(\Pi_j)}\left[\prod_{i=j+1}^d \widehat{\mathds{P}}(x_i \mid \Pi_i)\right],
\end{align}
with probability at least $1-\delta$. Note that Assumption \ref{structure} ensures that $x_d$ does not belong to $\Pi_i$ for any $i$. Therefore summing (\ref{tv-3}) over $x_d$ shows that
\begin{align*}
&\norm{\widehat{\mathds{P}}-\mathds{P}}_{\mathrm{TV}}\nonumber \\
&\qquad \leq \sum_{x_1=0,1}\ldots\sum_{x_{d-1}=0,1}\sum_{j=1}^d \left[\prod_{i=1}^{j-1} {\mathds{P}}(x_i \mid \Pi_i)\right] \frac{6d 2^{k} (k+1)}{n \epsilon}\log \frac{2d}{\delta}\nonumber \\
&\qquad \times \frac{1}{\mathds{P}(\Pi_j)}\left[\prod_{i=j+1}^{d-1} \widehat{\mathds{P}}(x_i \mid \Pi_i)\right].
\end{align*}
By Assumption \ref{structure} and induction, we can sum (\ref{tv-3}) over $x_{d-1},x_{d-2},\cdots,x_{j+1}$ and get
\begin{align}\label{tv-4}
&\norm{\widehat{\mathds{P}}-\mathds{P}}_{\mathrm{TV}}\nonumber \\
&\qquad \leq\sum_{j=1}^d \sum_{x_1=0,1}\ldots\sum_{x_j=0,1} \left[\prod_{i=1}^{j-1} {\mathds{P}}(x_i \mid \Pi_i)\right] \frac{6 d 2^{k} (k+1)}{n \epsilon}\log \frac{2d}{\delta}\frac{1}{\mathds{P}(\Pi_j)}.
\end{align}
Note that by Assumption \ref{structure} and the definition of Bayesian network, the product $\prod_{i=1}^{j-1} {\mathds{P}}(x_i \mid \Pi_i)$ is exactly the joint probability $\mathds{P}(x_1,\cdots,x_{j-1})$. Then the right-hand side of (\ref{tv-4}) is bounded as
\begin{align}\label{tv-5}
&\sum_{j=1}^d \sum_{x_1,\cdots,x_j} \mathds{P}(x_1,\cdots,x_{j-1}) \frac{6 d 2^{k} (k+1)}{n \epsilon}\log \frac{2d}{\delta}\frac{1}{\mathds{P}(\Pi_j)}\nonumber \\
& \qquad \leq \sum_{j=1}^d \sum_{x_j,\Pi_j} \sum_{A}\mathds{P}(x_1,\cdots,x_{j-1})\frac{6 d 2^{k} (k+1)}{n \epsilon}\log \frac{2d}{\delta}\frac{1}{\mathds{P}(\Pi_j)}\nonumber \\
& \qquad \leq \sum_{j=1}^d \sum_{x_j,\Pi_j}\frac{12d 2^{k} (k+1)}{n \epsilon}\log \frac{2d}{\delta}.
\end{align}
Here in the second line we decompose the summation into two parts: summing over $\Pi_j$ and summing over the rest nodes $A:=\{x_1,\cdots,x_{j-1}\}- \Pi_j$. The third line is due to the fact that \[
\sum_{A}\mathds{P}(x_1,\cdots,x_{j-1})=\mathds{P}(\Pi_j).
\]
By Assumption \ref{sparsity}, the size of $\Pi_j$ is less than $k$. Therefore, summing over all the possible choices of $\Pi_j$ in (\ref{tv-5}) yields a less than $2^k$ factor. Combining (\ref{tv-4}) and (\ref{tv-5}) and summing over $j=1,\ldots,d$, we then prove Theorem \ref{PrivBayes tv}.
\end{proof}
\section{Proof of Theorem \ref{PrivBayes L2}}\label{L2-proof}
We begin with a technical result that characterizes the $L^2$-projection post process.
\begin{lemma}\label{noisy L2}
For a distribution $(a_1,\cdots,a_m)$, we denote its outcome after adding i.i.d.\  ${\mathrm{Lap}}(d/n\epsilon)$ and $L^2$-projection as $(b_1,\cdots,b_m)$. Then for all large $n$, it holds that
\[
\norm{(a_1,\cdots,a_m)-(b_1,\cdots,b_m)}_{L^2}\leq \frac{\sqrt{m}d}{n\epsilon}\log \frac{m}{\delta},
\]
with probability at least $1-\delta$.
\end{lemma}

\begin{proof}
We denote the outcome after adding i.i.d.\ Laplace noise as $(a_1+v_1,\cdots, a_m+v_m)$, where $v_i\sim$ ${\mathrm{Lap}}(d/n\epsilon)$. Then the definition of $L^2$-projection shows that
\begin{align}\label{pre-L2-1}
&\norm{(b_1,\cdots,b_m)-(a_1,\cdots,a_m)}_{L^2}\leq \nonumber \\ &\qquad\norm{(a_1+v_1,\cdots,a_m+v_m)-(a_1,\cdots,a_m)}_{L^2}.
\end{align}
A simple union bound over $m$ Laplace random variables shows that
\[
\max _i\abs{v_i}\leq \frac{d}{n\epsilon} \log (\frac{m}{\delta}):=M,
\]
with probability at least $1-\delta$. Combining this fact with (\ref{pre-L2-1}) proves Lemma \ref{noisy L2}.
\end{proof}
The next one is the $L^2$-version of Lemma \ref{conditional-lemma}.
\begin{lemma}\label{conditional-L2}
For two non-negative real vectors $(a_1,\cdots,a_m)$ and $(b_1,\cdots,b_m)$ (not necessary to be distributions), if for some $\beta\geq 0$,
\begin{equation}\label{conditional-L2-2}
\sum_j\abs{a_j-b_j}^2\leq \beta,
\end{equation}
then for any $l\in \{1,\cdots,m\}$, the following result holds
\begin{equation}\label{conditional-L2-3}
\abs{\frac{a_l}{\sum_{j=1}^m a_j}-\frac{b_l}{\sum_{j=1}^m b_j}}^2\leq \frac{2m\beta}{(\sum_{j=1}^m b_j)^2}.
\end{equation}
\end{lemma}
\begin{proof}
The left-hand side of (\ref{conditional-L2-3}) can be rewritten as
\begin{equation}\label{conditional-L2-4}
\frac{\abs{a_l \sum_{j=1}^m b_j-b_l \sum_{j=1}^m a_j}^2}{(\sum_{j=1}^m a_j)^2(\sum_{j=1}^m b_j)^2}.
\end{equation}
We denote the numerator as $S_1$. Then by adding and contracting $\sum_{j=1}^m a_l a_j$, we rewrite and bound $S_1$ as
\begin{align}\label{conditional-L2-5}
&\abs{\sum_{j=1,j\neq l}^m a_l(b_j-a_j)+a_j(a_l-b_l)}^2\nonumber \\
&\qquad\leq 2m(\sum_{j=1,j\neq l}^m(a_j-b_j)^2 a_l^2+ (a_l-b_l)^2 \sum_{j=1,j\neq l}^m a_j^2) \nonumber\\
&\qquad\leq 2m (\sum_{j=1}^m (a_j-b_j)^2)(\sum_{j=1}^m a_j^2).
\end{align}
Here the second line is due to Cauchy-Schwartz inequality. Combining (\ref{conditional-L2-2}) and (\ref{conditional-L2-5}) yields
\[
S_1\leq 2m (\sum_j (a_j-b_j)^2)(\sum_{j}a_j)^2\leq 2m \beta (\sum_{j}a_j)^2,
\]
which concludes the proof.
\end{proof}

With the above preparation, we are now ready to prove Theorem \ref{PrivBayes L2}.
\begin{proof}
By Assumption \ref{structure} we have
\begin{align}\label{L2-1}
&\norm{\widehat{\mathds{P}}-\mathds{P}}_{L^2}^2\nonumber \\
&\qquad =\sum_{x_1=0,1}\ldots\sum_{x_d=0,1}
\underbrace{\abs{\prod_{i=1}^d \widehat{\mathds{P}}(x_i\mid \Pi_i)-\prod_{i=1}^d\mathds{P}(x_i\mid\Pi_i)}^2}_{ \textrm{term (ii)}}.
\end{align}
By adding and subtracting  $d-1$ terms, we rewrite term (ii) in the parentheses as
\begin{align}\label{L2-2}
&\textrm{Term (ii)}\nonumber \\
&\qquad \leq \abs{\sum_{j=1}^d \left[\prod_{i=1}^{j-1} {\mathds{P}}(x_i \mid \Pi_i)\right] \abs{\widehat{\mathds{P}}(x_j\mid \Pi_j)-\mathds{P}(x_j \mid \Pi_j)}
\left[\prod_{i=j+1}^d \widehat{\mathds{P}}(x_i \mid \Pi_i)\right]}^2.
\end{align}
Applying Cauchy-Schwartz inequality yields that
\begin{align}\label{L2-3}
&\textrm{Term (ii)}\nonumber \\
&\qquad \leq d \sum_{j=1}^d \left[\prod_{i=1}^{j-1} {\mathds{P}}(x_i \mid \Pi_i)\right]^2 \abs{\widehat{\mathds{P}}(x_j\mid \Pi_j)-\mathds{P}(x_j \mid \Pi_j)}^2
\left[\prod_{i=j+1}^d \widehat{\mathds{P}}(x_i \mid \Pi_i)\right]^2.
\end{align}
Since Assumption \ref{structure} ensures that $x_d$ does not belong to $\Pi_i$ for any $i$ , we again sum \textrm{term (ii)} over $x_d$. The right-hand side of (\ref{L2-1}) can be then rewritten as
\begin{align}\label{l2-4}
&d \sum_{x_1=0,1}\ldots\sum_{x_{d-1}=0,1}\sum_{j=1}^d \left[\prod_{i=1}^{j-1} {\mathds{P}}(x_i \mid \Pi_i)\right]^2 \nonumber \\
&\qquad \times\abs{\widehat{\mathds{P}}(x_j\mid \Pi_j)-\mathds{P}(x_j \mid \Pi_j)}^2
\left[\prod_{i=j+1}^{d-1} \widehat{\mathds{P}}(x_i \mid \Pi_i)\right]^2.
\end{align}
Here we use the fact that
\[
\sum_{x_d=0,1}\widehat{\mathds{P}}(x_d\mid \Pi_d)^2\leq
\left(\sum_{x_d=0,1}\widehat{\mathds{P}}(x_d\mid \Pi_d)\right)^2=1.
\]
Therefore, by summing (\ref{L2-3}) over $x_d,\cdots, x_{j+1}$ and applying induction, the left-hand side of (\ref{L2-1}) is bounded as
\begin{align}\label{L2-5}
&\norm{\widehat{\mathds{P}}-\mathds{P}}^2_{L^2}\nonumber \\
&\qquad \leq d \sum_{j=1}^d\sum_{x_1=0,1}\ldots\sum_{x_{j}=0,1} \left[\prod_{i=1}^{j-1} {\mathds{P}}(x_i \mid \Pi_i)\right]^2 \abs{\widehat{\mathds{P}}(x_j\mid \Pi_j)-\mathds{P}(x_j \mid \Pi_j)}^2\nonumber \\
& \qquad \leq d \sum_{j=1}^d\sum_{x_1=0,1}\ldots\sum_{x_{j}=0,1}
{\mathds{P}}(x_1,\cdots,x_{j-1})^2   \abs{\widehat{\mathds{P}}(x_j\mid \Pi_j)-\mathds{P}(x_j \mid \Pi_j)}^2 \nonumber \\
& \qquad \leq d \sum_{j=1}^d \sum_{x_j, \Pi_j} \mathds{P}(\Pi_j)^2 \abs{\widehat{\mathds{P}}(x_j\mid \Pi_j)-\mathds{P}(x_j \mid \Pi_j)}^2.
\end{align}
Here the third line is due to Assumption \ref{structure}. We now explain the last line of (\ref{L2-5}). Denote the set $A:=\{x_i\mid i\in (1,\cdots,j-1), x_i\notin \Pi_j\}$. Then the last line is derived by the following fact
\begin{align*}
\sum_{A} {\mathds{P}}(x_1,\cdots,x_{j-1})^2  \leq (\sum_{A} {\mathds{P}}(x_1,\cdots,x_{j-1}))^2=\mathds{P}(\Pi_j)^2.
\end{align*}
By Lemma \ref{noisy L2} and Lemma \ref{conditional-L2}, the right-hand side in (\ref{L2-5}) is bounded by
\begin{align*}
&d \sum_{j=1}^d \sum_{x_j, \Pi_j} \mathds{P}(\Pi_j)^2 2^k \frac{d^2}{(n\epsilon)^2}(\log \frac{2^{k+1} d}{\delta})^2 \frac{1}{\mathds{P}(\Pi_j)^2}\\
& \qquad \leq\frac{62^{2k}d^4}{(n\epsilon)^2}(\log \frac{2^{k+1} d}{\delta})^2.
\end{align*}

 Taking square root on both sides proves Theorem \ref{PrivBayes L2}.
\end{proof}
\section{Proof of Theorem \ref{PrivBayes utility}}\label{utility-proof}
We begin with Lemma \ref{utility-tv} that characterizes the bound of term (i) and term (vi).

\begin{lemma}\label{utility-tv}
If Assumption \ref{boundedness} holds, then for any $\theta\in \mathcal{C}$, the difference between $R(\theta,\mathds{Q})$ and $R(\theta,\mathds{P})$ is bounded as,
\[
\abs{R(\theta,\mathds{Q})-R(\theta,\mathds{P})}\leq 2 \norm{\mathds{P}-\mathds{Q}}_{\mathrm{TV}}.
\]
\end{lemma}
Here $\mathds{P}$ is the empirical distributions of $\mathds{D}$ and $\mathds{Q}$ is the output of PrivBayes. They are distributions on $\Omega$.
\begin{proof}
Recall the definition of $R(\cdot,\cdot)$ in (\ref{r-define}), the left-hand side can be written as
\[
\abs{R(\theta,\mathds{Q})-R(\theta,\mathds{P})}\leq
\sum_{x\in\Omega} \ell(\theta,x)^2 \abs{\mathds{P}(x)-\widehat{\mathds{P}}(x)}.
\]
Directly applying Assumption \ref{boundedness} concludes the proof.
\end{proof}

The next lemma is from the standard Rademacher analysis. We omit the proof.
\begin{lemma}\label{rademarcher}
For a distribution $\mathds{Q}$ on $\Omega$, we draw  $\widehat{n}$ i.i.d.\  samples from $\mathds{Q}$. The empirical distribution of these samples is denoted as $\widehat{\mathds{Q}}$. Then if Assumption \ref{boundedness} holds, with probability at least $1-\delta$, we have
\[
\sup_{\theta\in \mathcal{C}}\abs{R(\theta,\mathds{Q})-R(\theta,\widehat{\mathds{Q}})}
\leq 2\mathcal{R}_{C}+\sqrt{\frac{\log \frac1\delta}{2\widehat{n}}}.
\]
Here $\mathcal{R}_C$ is the Rademacher complexity of the function class $\{x\in \Omega\rightarrow \ell(\theta,x) \mid \theta\in \mathcal{C}\}$.
\end{lemma}
With Lemma \ref{utility-tv} and Lemma \ref{rademarcher}, term (iv) is now bounded as follows.
\begin{lemma}\label{utility-3}
If Assumption \ref{boundedness} holds, then with probability at least $1-\delta$, the following bound holds
\[
\textrm{term (iv)}
\leq 4\norm{\mathds{P}-{\mathds{Q}}}_{\mathrm{TV}}+4\mathcal{R}_{C}+2\sqrt{\frac{\log \frac1\delta}{2\widehat{n}}}.
\]
\end{lemma}
\begin{proof}
First note the definition of $\theta^{*}_{\mathrm{syn}}$ ensures that $R(\theta^*_{\mathrm{syn}},\widehat{\mathds{Q}})\leq R(\theta^*,\widehat{\mathds{Q}})$. By adding and subtracting the same term, we have
\begin{align}\label{term4-1}
&\textrm{term (iv)}= \nonumber\\
&\qquad (R(\theta^*,\widehat{\mathds{Q}})-R(\theta^*,{\mathds{Q}}))
+(R(\theta^*_{\mathrm{syn}},\mathds{Q})-{R}(\theta^*_{\mathrm{syn}},\widehat{\mathds{Q}}))\nonumber\\
&\qquad+R(\theta^*,\mathds{Q})-R(\theta^*_{\mathrm{syn}},\mathds{Q}).
\end{align}
Moreover, note that $R(\theta^*,\mathds{P})\leq R(\theta^*_{\mathrm{syn}},\mathds{P})$, the right-hand side in (\ref{term4-1}) can be further upper bounded by
\begin{align}\label{term4-2}
&\textrm{term (iv)}= \nonumber\\
&\qquad (R(\theta^*,\widehat{\mathds{Q}})-R(\theta^*,{\mathds{Q}}))
+(R(\theta^*_{\mathrm{syn}},\mathds{Q})-{R}(\theta^*_{\mathrm{syn}},\widehat{\mathds{Q}}))\nonumber\\
&\qquad+(R(\theta^*,\mathds{Q})-R(\theta^*,\mathds{P}))
+(R(\theta^*_{\mathrm{syn}},\mathds{P})-R(\theta^*_{\mathrm{syn}},\mathds{Q})).
\end{align}
Applying Lemma \ref{utility-tv} and Lemma \ref{rademarcher} concludes the proof.
\end{proof}
We are now ready to prove Theorem \ref{PrivBayes utility}. Note that when $\lambda=0$, the following equation holds,
\[
\theta^*=\widehat{\theta}, \, \theta^*_{\mathrm{syn}}=\widehat{\theta}_{\mathrm{syn}}.
\]
Therefore term (iii) and term (vii) in (\ref{decompose}) vanish when there is no regularization ($\lambda=0$). We combine them to be the term $C(\lambda)$. Further applying Lemma \ref{utility-tv}, Lemma \ref{rademarcher} and Lemma \ref{utility-3} to (\ref{decompose}) concludes the proof of Theorem \ref{PrivBayes utility}.
\section{Proof of Theorem \ref{tv-lower}}\label{lower-proof}
Our proof is based on packing technique. The key step is constructing a family of datasets that are ``spread out" enough. Without loss of generality, we assume $\alpha=\log (\delta|\Omega|)/\epsilon$ is an integer. We first fix an element $e$ in $\Omega$. Then for any $x\in \Omega$ and $x\neq e$, we construct a dataset of size $n$ corresponding to $x$ as follows,
\[
\mathds{D}_x:=\{n-\alpha \text{ copies of }e, \, \alpha \text{ copies of }x\}.
\]
Here we require that $\alpha\leq n$. The corresponding empirical distribution is denoted as $\mathds{P}_x$. One can easily verify the following facts.
\begin{enumerate}[i.]
\item
For any two different elements $x,y$ in $\Omega$, the dataset $\mathds{D}_x$ differs from $\mathds{D}_y$ in exactly $\alpha $ elements,
\item
For any two different elements $x,y$ in $\Omega$, the TV distance between $\mathds{P}_x$ and $\mathds{P}_y$ is exactly $2\alpha /n$.
\end{enumerate}
Suppose that for any $x\in \Omega-e$, with probability at least $\beta:=2\exp(\epsilon\alpha)/\abs{\Omega}$, we have
\[
\norm{A(\mathds{D}_x)-\mathds{P}_x}_{\mathrm{TV}}\leq \frac{\alpha }{2n}.
\]
For a fixed $z$, we define the event
\[
\mathcal{B}(y)=\left\{\norm{A(\mathds{D}_z)-\mathds{P}_y}_{\mathrm{TV}}\leq \frac{\alpha }{2n}\right\},
\]
for any $y\in \Omega-e-z$. Since $A(\cdot)$ is $\epsilon$-differentially private, it holds that the probability of $\mathcal{B}(y)$ is at least $\exp(-\alpha \epsilon)\cdot\beta$. Here we use Fact i mentioned above. Moreover, Fact ii ensures that the events $\{\mathcal{B}(y)\}$ are mutually disjoint for different $y$. Summing up all the disjoint events implies
\[
(\abs{\Omega}-2)\exp(-\alpha \epsilon)\cdot\beta\leq \sum_{y\in \Omega-e-z}\mathrm{Pr}[\mathcal{B}(y)]\leq 1.
\]
This leads to a contradiction since $\beta=2\exp(\epsilon\alpha)/\abs{\Omega}$, which implies that there exists a $x\in \Omega$ such that
\[
\norm{A(\mathds{D}_x)-\mathds{P}_x}_{\mathrm{TV}}\geq \frac{\alpha }{2n}
\]
with probability at least $1-\beta$. This concludes the proof of Theorem \ref{tv-lower}.

% \subsection{Margins in LaTeX}

% Most of the margin problems come from figures positioned by hand using
% \verb+\special+ or other commands. We suggest using the command
% \verb+\includegraphics+
% from the graphicx package. Always specify the figure width as a multiple of
% the line width as in the example below using .eps graphics
% \begin{verbatim}
%   \usepackage[dvips]{graphicx} ...
%   \includegraphics[width=0.8\linewidth]{myfile.eps}
% \end{verbatim}
% or % Apr 2009 addition
% \begin{verbatim}
%   \usepackage[pdftex]{graphicx} ...
%   \includegraphics[width=0.8\linewidth]{myfile.pdf}
% \end{verbatim}
% for .pdf graphics.
% See section~4.4 in the graphics bundle documentation (\url{http://www.ctan.org/tex-archive/macros/latex/required/graphics/grfguide.ps})

% A number of width problems arise when LaTeX cannot properly hyphenate a
% line. Please give LaTeX hyphenation hints using the \verb+\-+ command.

% \subsubsection*{Author Contributions}
% If you'd like to, you may include  a section for author contributions as is done
% in many journals. This is optional and at the discretion of the authors.

% \subsubsection*{Acknowledgments}
% Use unnumbered third level headings for the acknowledgments. All
% acknowledgments, including those to funding agencies, go at the end of the paper.

% \appendix
% \section{Appendix}
% You may include other additional sections here.

\end{document}